\documentclass[letterpaper, 10 pt, conference]{ieeeconf}  

\IEEEoverridecommandlockouts                              

\pdfoutput=1
\usepackage[T1]{fontenc}
\usepackage[english]{babel}
\usepackage[pdftex]{graphicx}   
\usepackage{epsfig}             
\usepackage{amsmath}            
\usepackage{amssymb}            
\usepackage{mathtools}
\usepackage{multicol}
\usepackage[bookmarks=true]{hyperref}
\usepackage{xcolor}
\usepackage{siunitx}
\usepackage{lipsum}
\usepackage{scrextend}

\usepackage{bm}

\usepackage[T1]{fontenc}

\usepackage{MnSymbol}
\usepackage{mathdots}

\usepackage{amsthm}

\providecommand{\R}{\mathbb{R}}

\providecommand{\SE}{\mathbf{SE}}

\providecommand{\grpG}{\mathbf{G}}

\providecommand{\gothg}{\mathfrak{g}}

\providecommand{\se}{\mathfrak{se}}

\providecommand{\calM}{\mathcal{M}}
\providecommand{\calN}{\mathcal{N}}

\providecommand{\vecL}{\mathbb{L}}

\providecommand{\PD}{\mathbb{S}_+} 

\DeclareMathOperator{\Ad}{Ad}
\DeclareMathOperator{\ad}{ad}

\providecommand{\tT}{\mathrm{T}} 
\providecommand{\td}{\mathrm{d}}

\providecommand{\Fr}[2]{\left. \mathrm{D}_{#1} \right|_{#2}}

\usepackage{accents}
\makeatletter
\providecommand{\scirc}{%
    \hbox{\fontfamily{\rmdefault}\fontsize{0.4\dimexpr(\f@size pt)}{0}\selectfont{\raisebox{-0.52ex}[0ex][-0.52ex]{$\circ$}}}}

\makeatother

\mathchardef\mhyphen="2D

\theoremstyle{plain}
\newtheorem{theorem}{Theorem}[section]

\newtheorem{lemma}[theorem]{Lemma}

\theoremstyle{definition}

\theoremstyle{remark}
\newtheorem{remark}[theorem]{Remark}

\newcommand{\AtoB}[2]{\;:\;#1\;\rightarrow\;#2}

\providecommand{\tT}{\mathrm{T}} 
\providecommand{\td}{\mathrm{d}}

\providecommand{\Fr}[2]{\mathrm{D}_{#1}\big|_{#2}}

\newcommand{\eye}{\mathbf{I}}

\newcommand{\Adsym}[2]{\mathrm{Ad}_{#1}\left[#2\right]}
\newcommand{\adsym}[2]{\mathrm{ad}_{#1}\left[#2\right]}

\newcommand{\AdMsym}[1]{\mathbf{Ad}^\vee_{#1}}
\newcommand{\adMsym}[1]{\mathbf{ad}^\vee_{#1}}

\newcommand{\sekntorsorsym}[2]{\mathcal{SE}_{#1}\left(#2\right)}

\newcommand{\Pose}[2]{\prescript{#1}{#2}{\mathbf{T}}}
\newcommand{\Rot}[2]{\prescript{#1}{#2}{\mathbf{R}}}
\newcommand{\Vector}[3]{\prescript{#1}{}{\bm{#2}}_{#3}}
\newcommand{\dotPose}[2]{\prescript{#1}{#2}{\dot{\mathbf{T}}}}

\newcommand{\dotVector}[3]{\prescript{#1}{}{\dot{\bm{#2}}}_{#3}}

\newcommand{\frameofref}[1]{\{$#1$\}}
\newcommand{\figref}[1]{Fig.~\ref{fig:#1}}

\newcommand{\equref}[1]{Equ.~(\ref{eq:#1})}

\newcommand{\tabref}[1]{Tab.~\ref{tab:#1}}
\newcommand{\cit}[1]{~\cite{#1}}

\usepackage{acro}
\usepackage{acro/acro}

\usepackage[caption=false,font=footnotesize]{subfig}

\usepackage{ctable}

\title{\LARGE \bf
Equivariant Filter Design for Inertial Navigation Systems with Input Measurement Biases 
}

\author{
Alessandro Fornasier$^{1}$, Yonhon Ng$^{2}$, Robert Mahony$^{2}$ and Stephan Weiss$^{1}$
\thanks{$^{1}$Alessandro Fornasier and Stephan Weiss are with the Control of Networked Systems Group, University of Klagenfurt, Austria. {\tt\small \{alessandro.fornasier,stephan.weiss\}@ieee.org} }
\thanks{$^{2}$Yonhon Ng and Robert Mahony are with the System Theory and Robotics Lab, Australian National University, Australia. {\tt\small \{yonhon.ng,robert.mahony\}@anu.edu.au}}
\thanks{This work is supported by the EU-H2020 project BUGWRIGHT2 (GA 871260)}%
\thanks{\textbf{Accepted January/2022 for ICRA 2022, DOI follows ASAP~\copyright IEEE.}}
}

\hypersetup{
  pdfinfo={
    Title={Equivariant Filter Design for Inertial Navigation Systems with Input Measurement Biases},
    Author={Alessandro Fornasier, Yonhon Ng, Robert Mahony and Stephan Weiss},
    Subject={Equivariant Filter},
    Keywords={Equivariant Filter, Equivariant System, State Estimation, Navigation, EQF, IEKF, Bias}
  }
}

\begin{document}

\maketitle
\thispagestyle{empty}
\pagestyle{empty}


\begin{abstract}
Inertial Navigation Systems (INS) are a key technology for autonomous vehicles applications. 
Recent advances in estimation and filter design for the INS problem have exploited geometry and symmetry to overcome limitations of the classical Extended Kalman Filter (EKF) approach that formed the mainstay of INS systems since the mid-twentieth century. 
The industry standard INS filter, the Multiplicative Extended Kalman Filter (MEKF), uses a geometric construction for attitude estimation coupled with classical Euclidean construction for position, velocity and bias estimation. 
The recent Invariant Extended Kalman Filter (IEKF) provides a geometric framework for the full navigation states, integrating attitude, position and velocity, but still uses the classical Euclidean construction to model the bias states. 
In this paper, we use the recently proposed Equivariant Filter (EqF) framework to derive a novel observer for biased inertial-based navigation in a fully geometric framework.  
The introduction of virtual velocity inputs with associated virtual bias leads to a full equivariant symmetry on the augmented system. 
The resulting filter performance is evaluated with both simulated and real-world data, and demonstrates increased robustness to a wide range of erroneous initial conditions, and improved accuracy when compared with the industry standard Multiplicative EKF (MEKF) approach. 
\end{abstract}
\section{Introduction}

State observers for mobile systems are core components in enabling autonomous robotic applications. Classical solutions including \ac{ekf}, \ac{ukf}, and non-linear optimization methods demonstrated their relevance in practical application. However, despite their success, they are known to lack in robustness due to their local coordinate representation and linearization issues. In fact, the mathematical proof of their (guaranteed) convergence has only been shown under several strong assumptions\cit{krener2004local}\cit{7523335}.
This has motivated the study of geometric observers that exploit symmetry and geometric properties of systems on homogeneous spaces and Lie groups to deliver better robustness and performance. 
Geometric nonlinear observers for navigation problem and pose estimation on $\SE(3)$ have been derived in deterministic form\cit{baldwin2009nonlinear}, and stochastic form\cit{Hashim2020NonlinearMeasurements}\cit{Hashim2021NonlinearPerformance}.
However, most of these observers are designed for first order kinematic systems where a velocity input is available. 
Motivated by the widespread use of robotic systems carrying a low-cost \ac{imu}, and with particular focus on the non-biased inertial navigation problem, nonlinear geometric deterministic\cit{ng2019attitude}\cit{Ng2020PoseKinematics} and stochastic\cit{Hashim2021GeometricFusion} observers for second order kinematic systems have been considered. Different from these approaches, Barrau and Bonnabel \cit{7523335} proposed the \ac{iekf}, a geometric \ac{ekf}-like observer for kinematic systems posed on a matrix Lie group with invariance properties. 
Reduced linearisation error and local convergence properties are key features that make the \ac{iekf} an excellent choice for a broad class of (group affine) systems on matrix Lie groups.

Asymptotically stable observers for kinematic systems on homogeneous spaces or Lie groups require measurements of the system input. 
In practice, such measurements are typically corrupted by constant or slowly time-varying unknown biases. 
The problem of observer design for concurrent estimation of the system state and input measurement biases has been tackled by different authors with different methodologies, including; non-linear deterministic observers\cit{khosravian2015observers}\cit{Ludher2021NonlinearPerformance}, gradient-based observers\cit{zlotnik2018gradient} and \ac{iekf} observers\cit{barrau:tel-01247723}\cit{doi:10.1177/0278364919894385}. 
Common to all prior approaches, however, is that the inclusion of the bias state in the estimation breaks the symmetry and leads to complete or partial loss of the key properties such observers are characterized by.

In this paper, we tackle the problem of designing an \ac{eqf} for the \acl{bins}. 
Such a system is neither invariant nor group affine, and thus \ac{iekf} observers cannot be directly applied without loss of the log-linear property\cit{barrau:tel-01247723}\cit{doi:10.1177/0278364919894385}.
Motivated by the recent research in the new field of equivarariant system theory\cit{Mahony2020EquivariantGroups} and \ac{eqf} design\cit{VanGoor2020EquivariantSpaces} we use velocity extension and equivariant filter design methodologies to obtain a fully geometric \ac{ekf}-like filter for a navigation system including bias states. 
Key contributions of the paper are:
First, the exploitation of the semi-direct product structure to define the symmetry group ${\SE_2(3) \ltimes \se_2(3)}$, and the subsequent modeling of the \acl{bins} as equivariant system. 
Second, the presentation of a novel equivariant lift for the system onto the symmetry group leading to the proposed design of an \acl{eqf} for the lifted system.
To the best of our knowledge, this is the first work presenting a geometric EKF-like observer for the concurrent estimation of the navigation state and the input measurement biases characterized by a tightly-coupled geometric description of such states.
The performance of the proposed \ac{eqf} is demonstrated on simulations and real-world in-flight data of an \ac{uav} equipped with an \ac{imu} and able to receive extended pose measurements of position, attitude and velocity. A performance gain of $30\%$-$40\%$ in terms of input measurement bias estimation, as well as, enhanced stability and robustness compared to the industry standard M\ac{ekf}, make the proposed \ac{eqf} a suitable choice to replace other EKF-like observer for robotic navigation problem. Practical use cases of the proposed \ac{eqf} could include an aircraft equipped with multiple GNSS receivers able to obtain direct position and velocity measurements, and indirect attitude measurements given a sufficient baseline among the different receivers\cit{Pavlasek2021InvariantEstimation}, or an aircraft equipped with a single GNSS receiver that provides position and velocity measurements and a magnetometer to provide attitude information.

\section{Preliminaries}\label{sec:pre}

Let $\xi \in \calM$ be an element of a smooth manifold $\calM$, $\tT\calM$ denotes the tangent bundle, $\tT_{\xi}\calM$ denotes the tangent space of $\calM$ at $\xi$ and ${\mathfrak{X}\left(\calM\right)}$ denotes the set of all smooth vector fields on $\calM$.
%
%
A Lie group $\grpG$ is a smooth manifold endowed with a smooth group multiplication that satisfies the group axioms. For any $X, Y \in \grpG$, the group multiplication is denoted $XY$, the group inverse $X^{-1}$ and the identity element $I$. In this work we will focus our attention on a special case of Lie groups, called matrix Lie groups, which are those whose elements are invertible square matrices and where the group operation is the matrix multiplication.
For a given Lie group $\grpG$, the Lie algebra $\gothg$ is a vector space corresponding to $\tT_{I}\grpG$, together with a bilinear non-associative map ${[\cdot, \cdot] \AtoB{\gothg}{\gothg}}$ called Lie bracket.
%
The Lie algebra $\gothg$ is isomorphic with a vector space $\R^{n}$ of the same dimension ${n = \dim\left(\gothg\right)}$. We define the wedge map ${\left(\cdot\right)^{\wedge} \AtoB{\R^{n}}{\gothg}}$ and its inverse, the vee map ${\left(\cdot\right)^{\vee} \AtoB{\gothg}{\R^{n}}}$, as mappings between the vector space and the Lie algebra.

For any $X, Y \in \grpG$, the left and right translation by $X$ write
\begin{align*}
    &L_{X} \AtoB{\grpG}{\grpG}, \qquad L_{X}\left(Y\right) = XY ,\\
    &R_{X} \AtoB{\grpG}{\grpG}, \qquad R_{X}\left(Y\right) = YX ,
\end{align*} 
follows that for $\eta_{Y} \in \tT_{Y}\grpG$, their differential are defined by
\begin{align*}
    &\td L_{X} \AtoB{\tT_{Y}\grpG}{\tT_{XY}\grpG}, \qquad \td L_{X}\left[\eta_{Y}\right] = X\eta_{Y} ,\\
    &\td R_{X} \AtoB{\tT_{Y}\grpG}{\tT_{YX}\grpG}, \qquad \td R_{X}\left[\eta_{Y}\right] = \eta_{Y}X .
\end{align*} 
Proofs are provided in the supplementary material\cit{fornasier2022supplementary}.

The Adjoint map for $\grpG$, $\Ad_X \AtoB{\gothg}{\gothg}$ is defined by
\begin{equation*}
    \Adsym{X}{\bm{u}^{\wedge}} = \td L_{X} \td R_{X^{-1}}\left[\bm{u}^{\wedge}\right] ,
\end{equation*}
for every $X \in \grpG$ and ${\bm{u}^{\wedge} \in \gothg}$. In the context of matrix Lie group, we can define the Adjoint matrix to be the map ${\AdMsym{X} \AtoB{\R^{n}}{\R^{n}}}$ defined by
\begin{equation*}
    \AdMsym{X}\bm{u} = \left(\Adsym{X}{\bm{u}^{\wedge}}\right)^{\vee} .
\end{equation*}

In addition to the Adjoint map for the group $\grpG$, the adjoint map for the Lie algebra $\gothg$ can be defined as the differential at the identity of the Adjoint map for the group $\grpG$, therefore the adjoint map for the Lie algebra $\ad_{\bm{u}^{\wedge}} \AtoB{\gothg}{\gothg}$ is written
\begin{equation*}
    \adsym{\bm{u}^{\wedge}}{\bm{v}^{\wedge}} = \bm{u}^{\wedge}\bm{v}^{\wedge} - \bm{v}^{\wedge}\bm{u}^{\wedge} = \left[\bm{u}^{\wedge}, \bm{v}^{\wedge}\right] ,
\end{equation*}
that is equivalent to the Lie bracket. Again, in the context of matrix Lie group we can define the adjoint matrix $\adMsym{\bm{u}^{\wedge}} \AtoB{\R^{n}}{\R^{n}}$ as 
\begin{equation*}
    \adMsym{\bm{u}^{\wedge}}\bm{v} = \left(\bm{u}^{\wedge}\bm{v}^{\wedge} - \bm{v}^{\wedge}\bm{u}^{\wedge}\right)^{\vee} = \left[\bm{u}^{\wedge}, \bm{v}^{\wedge}\right]^{\vee} .
\end{equation*}

The Adjoint operator is a Lie-algebra automorphism 
\begin{equation*}
    \Adsym{X}{\adsym{\Vector{}{u}{}^{\wedge}}{\Vector{}{v}{}^{\wedge}}} = \adsym{\left(\Adsym{X}{\Vector{}{u}{}^{\wedge}}\right)}{\Adsym{X}{\Vector{}{v}{}^{\wedge}}} .
\end{equation*}

A right group action of a Lie group $\grpG$ on a differentiable manifold $\calM$ is a smooth map ${\phi \AtoB{\grpG\times \calM}{\calM}}$ that satisfies ${\phi\left(I, \xi\right) = \xi}$ and ${\phi\left(X, \phi\left(Y, \xi\right)\right) = \phi\left(YX, \xi\right)}$ for any $X,Y \in \grpG$ and $\xi \in \calM$.
A right group action $\phi$ induces a family of diffeomorphisms ${\phi_X \AtoB{\calM}{\calM}}$ and smooth nonlinear projections ${\phi_{\xi} \AtoB{\grpG}{\calM}}$. The group action $\phi$ is said to be transitive if the induced projection $\phi_{\xi}$ is surjective, and, in such case, $\calM$ is termed homogeneous space.
A right group action $\phi$ induces a right group action on $\mathfrak{X}\left(\calM\right)$, ${\Phi \AtoB{\grpG \times \mathfrak{X}\left(\calM\right)}{\mathfrak{X}\left(\calM\right)}}$ defined by the push forward\cit{Mahony2021ObserverEquivariance} such that
\begin{equation*}
\Phi\left(X,f\right) = \td\phi_{X}f \circ \phi_{X^{-1}},
\end{equation*}
for each $f \in \mathfrak{X}\left(\calM\right)$, $\phi$-invariant vector field and $X \in \grpG$.

\section{Biased Inertial Navigation System}\label{sec:sys}

This work considers the problem of estimating the extended pose\cit{barrau:tel-01247723} of a rigid body, together with the accelerometer and gyroscope biases that affect the \ac{imu} measurements. Therefore estimating the state of what we called \acl{bins}.
Let \frameofref{G} denote the global inertial frame of reference and \frameofref{I} denote the \ac{imu} frame, which is assumed to coincide with the  rigid body's center of mass. In non-rotating, flat earth assumption the (noise-free) system is written

\begin{subequations}\label{eq:bins}
    \begin{minipage}{0.625\linewidth}
        \begin{align}
            &\dot{\Rot{G}{I}} = \Rot{G}{I}\left(\Vector{I}{\bm{\bm{\omega}}}{} - \Vector{I}{b_{\bm{\omega}}}{}\right)^{\wedge} ,\\
            &\dot{\prescript{G}{}{\mathbf{p}}_{I}} = \Vector{G}{v}{I} ,\\
            &\dot{\Vector{G}{v}{I}} =  \Rot{G}{I}\left(\Vector{I}{a}{} - \Vector{I}{b_{a}}{}\right) + \Vector{G}{g}{} 
        \end{align}
    \end{minipage}
    \begin{minipage}{0.3\linewidth}
        \begin{align}
            &\dot{\Vector{I}{b_{\bm{\omega}}}{}} = \mathbf{0} ,\\
            &\dot{\Vector{I}{b_{a}}{}} = \mathbf{0} .
        \end{align}
    \end{minipage}
    \vspace{7.5pt}
\end{subequations}

Here $\Rot{G}{I}$ corresponds to the rigid body orientation and it is defined by a rotation matrix that rotates a vector $\Vector{I}{x}{}$ defined in the \frameofref{I} frame to a vector ${\Vector{G}{x}{} = \Rot{G}{I}\Vector{I}{x}{}}$ defined in the \frameofref{G} frame. The rigid body position and velocity expressed in the \frameofref{G} frame are denoted  $\Vector{G}{p}{I}$ and $\Vector{G}{v}{I}$ respectively. The gravity vector in the \frameofref{G} frame is denoted $\Vector{G}{g}{}$.  The body-fixed, biased, angular velocity and linear acceleration measurements provided by an \ac{imu} are denoted $\Vector{I}{\bm{\omega}}{}$ and $\Vector{I}{a}{}$. The gyroscope and accelerometer biases are denoted  $\Vector{I}{b_{\bm{\omega}}}{}$ and $\Vector{I}{b_{a}}{}$ respectively.

In order to exploit the geometrical properties of the system, let $\nu, \tau_{\omega}$, $\tau_{\nu}$, $\tau_{a}$ be, in general, virtual, non-physical inputs and $\Vector{I}{b_{\nu}}{}$ be an additional velocity-bias state variable. Then the system kinematics can be extended as
\begin{subequations}\label{eq:extendedbins}
    \begin{minipage}{0.625\linewidth}
        \begin{align}
            &\dot{\Rot{G}{I}} = \Rot{G}{I}\left(\Vector{I}{\bm{\omega}}{} - \Vector{I}{b_{\bm{\omega}}}{}\right)^{\wedge} ,\\
            &\dot{\prescript{G}{}{\mathbf{p}}_{I}} = \Rot{G}{I}\left(\nu - \Vector{I}{b_{\nu}}{}\right) + \Vector{G}{v}{I} ,\\
            &\dot{\Vector{G}{v}{I}} =  \Rot{G}{I}\left(\Vector{I}{a}{} - \Vector{I}{b_{a}}{}\right) + \Vector{G}{g}{} ,
        \end{align}
    \end{minipage}
    \begin{minipage}{0.3\linewidth}
        \begin{align}
            &\dot{\Vector{I}{b_{\bm{\omega}}}{}} = \tau_{\omega} ,\\
            &\dot{\Vector{I}{b_{\nu}}{}} = \tau_{\nu} ,\\
            &\dot{\Vector{I}{b_{a}}{}} = \tau_{a} .
        \end{align}
    \end{minipage}
    \vspace{7.5pt}
\end{subequations}

Where it is clear that fixing the additional state variable and virtual inputs to zero lets us retrieve the original kinematics in \equref{bins}.
For the sake of clarity we have used bold letters for state variables and modeled inputs and non-bold letters without superscripts for virtual, non-physical inputs.

Let $\Pose{G}{I} = \left(\Rot{G}{I},\,\Vector{G}{p}{I},\,\Vector{G}{p}{I}\right) \in \sekntorsorsym{2}{3}$ denote platform's extended pose, and ${\Vector{I}{b}{} = \left(\Vector{I}{b_{\bm{\omega}}}{},\,\Vector{I}{b_{\nu}}{},\,\Vector{I}{b_{a}}{}\right) \in \R^{9}}$ the modeled measurement biases. Let ${\bm{u} = \left(\Vector{I}{w}{}^{\wedge}, \Vector{}{g}{}^{\wedge}, \tau^{\wedge} \right) \in \mathbb{L}}$, ${\mathbb{L} \subseteq \se_2(3) \times \se_2(3) \times \se_2(3)}$ denote the system input, where ${\Vector{I}{w}{} = \left(\Vector{I}{\bm{\omega}}{},\,\nu,\,\Vector{I}{a}{}\right) \in \R^{9}}$.
$\sekntorsorsym{2}{3}$ is the $\SE_2(3)$-Torsor\cit{Mahony2013ObserversSymmetry} and $\se_2(3)$ the Lie algebra of $\SE_2(3)$. The \acl{bins} can then be written as a first order kinematic system, with state ${\xi = \left(\Pose{G}{I}, \Vector{I}{b}{}^{\wedge}\right)}$ posed on the direct product manifold ${\calM \coloneqq \sekntorsorsym{2}{3} \times \se_2(3)}$, as
\begin{align}
    &\dotPose{G}{I} = f^{0}_1\left(\Pose{G}{I}\right)\Pose{G}{I} + \Pose{G}{I}\left(\Vector{I}{w}{}^{\wedge} - \Vector{I}{b}{}^{\wedge}\right) + \Vector{}{g}{}^{\wedge}\Pose{G}{I}  ,\\
    &\dotVector{I}{b}{}^{\wedge} = \tau^{\wedge} ,
\end{align}
where ${f^{0}_1 : \sekntorsorsym{2}{3} \to \tT\sekntorsorsym{2}{3}}$, ${f^{0}_1 \in \mathfrak{R}\left(\sekntorsorsym{2}{3}\right)}$ is a right invariant vector field defined by
\begin{equation}
    f^{0}_1\left(\Pose{G}{I}\right) \coloneqq \begin{bmatrix}
        \mathbf{0} & \Vector{G}{v}{I} & \mathbf{0}\\
        \mathbf{0} & \mathbf{0} & \mathbf{0}
    \end{bmatrix} .
\label{eq:f01}
\end{equation}

Moreover, the \acl{bins} kinematics can also be written in the following compact affine form:
\begin{equation}\label{eq:compactbins}
    \begin{split}
        \dot{\xi} &= f^{0}\left(\xi\right) + f_{\bm{u}}\left(\xi\right)\\
        &= f^{0}\left(\xi\right) + \left(\Pose{G}{I}\left(\Vector{I}{w}{}^{\wedge} - \Vector{I}{b}{}^{\wedge}\right) + \Vector{}{g}{}^{\wedge}\Pose{G}{I},\; \tau^{\wedge}\right) ,
    \end{split}
\end{equation}
with ${f^{0}\left(\xi\right) \coloneqq \left(f^{0}_1\left(\Pose{G}{I}\right) \Pose{G}{I},\; \mathbf{0}^{\wedge}\right)}$, ${f^{0} \in \mathfrak{R}\left(\calM\right)}$ a right invariant vector field, termed drift field.

Along this work, we consider the case where an extended pose measurement of the rigid body is available to the system. In such cases the output space can be defined by ${\calN \coloneqq \sekntorsorsym{2}{3}}$ and, thus, the configuration output ${h \AtoB{\calM}{\calN}}$ is defined by
\begin{equation}\label{eq:confout}
    h\left(\xi\right) = \Pose{G}{I} .
\end{equation}

Then a real-world measurement with associated Gaussian noise ${\Vector{}{n}{}}$ can be modeled as\cit{Brossard2021AssociatingEarth}
\begin{equation}\label{eq:real_meas}
    y = \Pose{G}{I}\exp\left(\Vector{}{n}{}^\wedge\right) \quad \text{or} \quad y = \exp\left(\Vector{}{n}{}^\wedge\right)\Pose{G}{I}.
\end{equation}
depending if uncertainties are associated locally or globally.
\section{Symmetry of the Biased Inertial Navigation System}\label{sec:geo}

Physical systems in robotics usually carry natural symmetries that encodes invariance or equivariance properties of their dynamical models under space transformations given by a symmetry group. As already recognized and widely accepted, symmetric properties of kinematic systems can be exploited to design high-performance and robust observers. In particular, as mentioned in the introduction, a key contribution of this paper is the exploitation of a novel symmetry group ${\SE_2(3)_{\se_2(3)}^{\ltimes} = \SE_2(3) \ltimes \se_2(3)}$ as a semi-direct product of $\SE_2(3)$ with $\se_2(3)$, which leads to a new formulation of the \acl{bins} as an equivariant kinematic system. Additional mathematical preliminaries, as well as, extended proofs are available in the supplementary material\cit{fornasier2022supplementary}.

For the sake of clarity, in the two following sections we omit all the superscript and subscript associated with reference frames from the state variables and input to improve the readability of the definitions and theorems. Therefore, we consider ${\xi = \left(\Pose{}{}, \Vector{}{b}{}^{\wedge}\right) \in \calM}$ and ${\Vector{}{u}{} = \left(\Vector{}{w}{}^{\wedge}, \Vector{}{g}{}^{\wedge}, \tau^{\wedge}\right) \in \vecL}$.

\subsection{The Semi-direct Product Group}

Let $X = \left(A, a\right)$ and $Y = \left(B, b\right)$ be elements of the symmetry group $\SE_2(3)_{\se_2(3)}^{\ltimes}$.
The semi-direct group product is defined by ${YX = \left(BA, b + \Adsym{B}{a}\right)}$.
The inverse element is ${X^{-1} = \left(A^{-1}, -\Adsym{A^{-1}}{a}\right)}$ with identity element ${\left(I, 0\right)}$.
The inverse of a product is ${\left(YX\right)^{-1} = \left(\left(BA\right)^{-1}, -\Adsym{\left(BA\right)^{-1}}{b + \Adsym{B}{a}}\right)}$.

\subsection{Equivariance of the \acl{bins}}

\begin{lemma}
Define ${\phi \AtoB{\SE_2(3)_{\se_2(3)}^{\ltimes} \times \calM}{\calM}}$ as
\begin{equation}\label{eq:phi}
   \phi\left(X, \xi\right) \coloneqq \left(\Pose{}{}A,\; \Adsym{A^{-1}}{\Vector{}{b}{}^{\wedge} - a}\right) .
\end{equation}
Then, $\phi$ is a transitive right group action of $\SE_2(3)_{\se_2(3)}^{\ltimes}$ on $\calM$.
\end{lemma}
Proof is provided in the supplementary material\cit{fornasier2022supplementary}.

\begin{lemma}
Define ${\psi \AtoB{\SE_2(3)_{\se_2(3)}^{\ltimes} \times \vecL}{\vecL}}$ as
\begin{equation}\label{eq:psi}
    \begin{split}
        \psi\left(X,\bm{u}\right) 
        &\coloneqq \left(\Adsym{A^{-1}}{\Vector{}{w}{}^{\wedge} - a} + {f^{0}_1\left(A^{-1}\right)},\; \Vector{}{g}{}^{\wedge},\; \Adsym{A^{-1}}{\tau^{\wedge}}\right) ,
    \end{split}
\end{equation}
where $f^0_1 : \SE_2(3) \rightarrow \tT \SE_2(3)$ is given by \equref{f01}.  Then, $\psi$ is a right group action of $\SE_2(3)_{\se_2(3)}^{\ltimes}$ on $\vecL$.
\end{lemma}
Proof is provided in the supplementary material\cit{fornasier2022supplementary}.

\begin{theorem}
The \acl{bins} in \equref{compactbins} is equivariant under the actions $\phi$ in \equref{phi} and $\psi$ in \equref{psi} of the symmetry group $\SE_2(3)_{\se_2(3)}^{\ltimes}$. That is
\begin{equation*}
    f^0\left(\xi\right) + f_{\psi_{X}\left(\bm{u}\right)}\left(\xi\right) = \Phi_{X}f^0\left(\xi\right) + \Phi_{X}f_{\bm{u}}\left(\xi\right) .
\end{equation*}
\end{theorem}
Proof is provided in the supplementary material\cit{fornasier2022supplementary}.

\begin{remark}
Note that the re-modeling of the \acl{bins} from \equref{bins} to \equref{extendedbins} by adding additional virtual inputs and state variables is a necessary step to prove the equivariant property of the system.
This is one of our key contributions that allowed us
to find a symmetry of the \acl{bins} for the first time in the field of biased inertial navigation.
\end{remark}

\subsection{Output equivariance}
\begin{lemma}
Define ${\rho \AtoB{\SE_2(3)_{\se_2(3)}^{\ltimes} \times \calN}{\calN}}$ as
\begin{equation}
    \rho\left(X,y\right) \coloneqq yA .
\end{equation}
Then, the configuration output defined in \equref{confout} is equivariant\cit{VanGoor2020EquivariantSpaces}.
\end{lemma}
Proof is provided in the supplementary material\cit{fornasier2022supplementary}.

\section{Lifted System}\label{sec:lift}

\subsection{Equivariant Lift}

In order to design an \ac{eqf} we need to find a lift of the system kinematics into the Lie algebra of the symmetry group, that is a smooth map ${\Lambda \AtoB{\calM \times \vecL} \se_2(3)_{\se_2(3)}^{\ltimes}}$ linear in $\vecL$ such that
\begin{equation*}
    d\phi_{\xi}\left(\mathbf{I}\right)\left[\Lambda\left(\xi,\bm{u}\right)\right] =  f^0\left(\xi\right) + f_{\bm{u}}\left(\xi\right) ,
\end{equation*}
${\forall\; \xi \in \calM,\; \bm{u} \in \vecL}$. For the sake of clarity and due to the limited space, we split the lift function into two functions ${\Lambda_1, \Lambda_2 \AtoB{\calM \times \vecL} \se_2(3)}$.
\begin{theorem}
Define ${\Lambda_1 \AtoB{\calM \times \vecL} \se_2(3)}$ as
\begin{equation}\label{eq:lift1}
    \Lambda_1\left(\xi, \bm{u}\right) \coloneqq \left(\Vector{}{w}{}^{\wedge} - \Vector{}{b}{}^{\wedge}\right) + \Adsym{\Pose{}{}^{-1}}{\Vector{}{g}{}^{\wedge}} + \Pose{}{}^{-1}f^0_1\left(\mathbf{T}\right) .
\end{equation}
And, define ${\Lambda_2 \AtoB{\calM \times \vecL} \se_2(3)}$ as
\begin{equation}\label{eq:lift2}
    \Lambda_2\left(\xi, \bm{u}\right) \coloneqq \adsym{\Vector{}{b}{}^{\wedge}}{\Lambda_1\left(\xi, \bm{u}\right)} - \tau^{\wedge} .
\end{equation}
Then, the map ${\Lambda\left(\xi, \bm{u}\right) = \left(\Lambda_1\left(\xi, \bm{u}\right),\; \Lambda_2\left(\xi, \bm{u}\right)\right)}$ is an equivariant lift for the system in \equref{compactbins} with respect to the defined symmetry group.
\end{theorem}
Proof is provided in the supplementary material\cit{fornasier2022supplementary}.

\subsection{Origin}

The lift $\Lambda$ defined in \equref{lift1}-\eqref{eq:lift2} provides the necessary structure that connects the input space with the Lie algebra of the symmetry group and, therefore, defines a lifted system on the symmetry group. However, in order to do so, we need to define a state origin $\xi_{0} \in \calM$ for a global coordinate parametrization of the state space $\calM$ by the symmetry group $\SE_2(3)_{\se_2(3)}^{\ltimes}$ given by the projection ${\phi_{\xi_{0}} \AtoB{\SE_2(3)_{\se_2(3)}^{\ltimes}}{\calM}}$. The choice of the state origin is arbitrary, however, a clever choice that makes the following derivation easier is given by the identity of the symmetry group, such that ${\xi_{0} = \left(\eye, \Vector{}{0}{}^{\wedge}\right)}$.
Moreover, given the choice of the state origin $\xi_{0}$ to be the identity, the output origin is ${y_{0} = h\left(\xi_{0}\right) = \eye}$.

\subsection{Lifted System}

The lift $\Lambda$ together with the previously defined reference origin $\xi_{0}$ allow the construction of a lifted system on the symmetry group $\SE_2(3)_{\se_2(3)}^{\ltimes}$. Let $\Vector{}{u}{} \in \vecL$ be the system input and ${X = \left(A,a\right) \in \SE_2(3)_{\se_2(3)}^{\ltimes}}$ be the lifted system state, then, according to the general form ${\dot{X} = \td L_{X}\Lambda\left(\phi_{\xi_{0}}\left(X\right), \Vector{}{u}{}\right)}$ in\cit{Mahony2020EquivariantDesign}, and considering the choice of the origin, the lifted system kinematics for the \acl{bins} follow
\begin{align}
    &\dot{A} = A\left(\Vector{}{w}{}^{\wedge} + \Adsym{A^{-1}}{a}\right) + \Vector{}{g}{}^{\wedge}A + f^0_1\left(A\right) ,\\
    \begin{split}
        &\dot{a} = \mathrm{Ad}_{A}\left[\mathrm{ad}_{\left(-\Adsym{A^{-1}}{a}\right)}\left[\left(\Vector{}{w}{}^{\wedge} + \Adsym{A^{-1}}{a}\right)\right.\right.\\
        &\quad\:\left.\left. + \Adsym{A^{-1}}{\Vector{G}{g}{}^{\wedge}} + A^{-1}f^0_1\left(A\right)\right]- \tau^{\wedge}\right] .
    \end{split}
\end{align}
\section{Equivariant Filter Design}\label{sec:eqf}

Starting from this section we restore our original notation of state and input variables.

The \acl{eqf} design follows the steps described in\cit{VanGoor2020EquivariantSpaces}. Let ${\Lambda}$ be the equivariant lift defined in \equref{lift1}-\eqref{eq:lift2}, $\xi_{0}$ be the chosen identity state origin and ${\hat{X} \in \SE_2(3)_{\se_2(3)}^{\ltimes}}$ be the \acl{eqf} state, with initial condition ${\hat{X}\left(0\right) = \left(I, 0\right)}$. Then the state evolves according to
\begin{align*}
    \dot{\hat{X}} = \td L_{\hat{X}} \Lambda(\phi_{\xi_0}(\hat{X}), u) + \td R_{\hat{X}} \Delta, 
\end{align*}
such that
\begin{align}
    &\dot{\hat{A}} = \hat{A}\left(\Vector{I}{w}{}^{\wedge} + \Adsym{\hat{A}^{-1}}{\hat{a}}\right) + \Vector{G}{g}{}^{\wedge}\hat{A} + f^{0}_1\left(\hat{A}\right) + \Delta_1\hat{A} ,\\
    \begin{split}
       &\dot{\hat{a}} = \adsym{-\hat{a}}{\left(\Adsym{\hat{A}}{\Vector{I}{w}{}^{\wedge}} + \hat{a}\right) + \Vector{G}{g}{}^{\wedge} + f^0_1\left(\hat{A}\right)}\\
       &\quad\; - \Adsym{\hat{A}}{\tau^{\wedge}} + \Delta_2 + \adsym{\Delta_1}{\hat{a}} .
    \end{split}
\end{align}

The innovation term ${\Delta = \left(\Delta_1, \Delta_2\right)}$ is defined in the following analysis.

\subsection{Linearized Error and Output Dynamics}

Given the chosen identity state origin $\xi_0$ and the \acl{eqf} state $\hat{X}$, the state error and the filter error are defined to be respectively ${e = \phi_{\hat{X}}^{-1}\left(\xi\right)}$ and ${E = X\hat{X}^{-1}}$. 
Since we would like to have ${\hat{\xi} = \phi_{\hat{X}}\left(\xi_0\right) \rightarrow \xi}$, the \acl{eqf} goal is to drive the error ${e \rightarrow \xi_0}$.

To compute the linearized error dynamics we need to select a chart and a chart transition map ${\varepsilon \AtoB{\mathcal{U}_{\xi_0} \subset \calM}{\R^{18}}}$. If, similarly to what has been done for the lift and the innovation, we split the error $e \in \calM$ in the following two components ${e = \phi_{\hat{X}}^{-1}\left(\xi\right) = \left(e_1, e_2\right)}$.
Then the chart transition map can be defined as 
\begin{align}
\label{eq:local_coord_xi}
    \varepsilon\left(e\right) = \left(\epsilon_1, \epsilon_2\right) = \left(\log\left(e_1\right)^{\vee}, e_2^{\vee}\right) \in \R^{18} ,
\end{align}
and ${\varepsilon\left(\xi_0\right) = \mathbf{0} \in \R^{18}}$.

Therefore the linearized dynamics of $\varepsilon$ at $\mathbf{0}$ are
\begin{align}
    \dot{\varepsilon} &\approx \mathbf{A}_{t}^{0}\varepsilon - \Fr{e}{\xi_0}\varepsilon\left(e\right)\Fr{E}{\mathbf{I}}\phi_{\xi_0}\left(E\right)\left[\Delta\right] \label{eq:linearisation_epsilon} ,\\
    \mathbf{A}_{t}^{0} &= \Fr{e}{\xi_0}\varepsilon\left(e\right)\Fr{E}{\mathbf{I}}\phi_{\xi_0}\left(E\right)\Fr{e}{\xi_0}\Lambda\left(e, \Vector{}{u}{0}\right)\Fr{\varepsilon}{\mathbf{0}}\varepsilon^{-1}\left(\varepsilon\right) ,\label{eq:A}
\end{align}
where ${\Vector{}{u}{0} \coloneqq \psi\left(\hat{X}^{-1}, \Vector{}{u}{}\right)}$ is the origin input. The interested reader is referred to\cit{VanGoor2020EquivariantSpaces} for a complete and detailed derivation of the linearized error dynamics in \equref{linearisation_epsilon}. 

Similarly, the output can be linearized by selecting a chart and a chart transition map ${\delta \AtoB{\mathcal{U}_{y_0} \subset \mathcal{N}}{\R^{9}}}$ defined as
\begin{align}
\label{eq:local_coord_y}
    \delta\left(y\right) = \log\left(y\right)^{\vee}.
\end{align}
Then linearizing $\delta$ as function of the coordinates $\varepsilon$ about $\varepsilon = \mathbf{0}$ yields
\begin{align}
    \delta &\approx \mathbf{C}^{0}\varepsilon , \label{eq:linearisation_delta} \\
    \mathbf{C}^{0} &= \Fr{y}{y_0}\delta\left(y\right)\Fr{e}{\xi_0}h\left(e\right)\Fr{\varepsilon}{\mathbf{0}}\varepsilon^{-1}\left(\varepsilon\right) .\label{eq:C}
\end{align}

Solving the equations \eqref{eq:A} and \eqref{eq:C} for the linearized error state matrix ${\mathbf{A}_{t}^{0}}$ and the linearized output matrix ${\mathbf{C}^{0}}$ yields
\begin{align}
    \mathbf{A}_{t}^{0} &= \begin{bmatrix}
        \bm{\Upsilon} & -\mathbf{I}\\
        \mathbf{0} & \adMsym{\left(\Vector{I}{w}{0}^{\wedge} + \Vector{G}{g}{}^{\wedge}\right)}
    \end{bmatrix} \in \R^{18 \times 18} , \label{eq:At0}\\
    \bm{\Upsilon} &= \begin{bmatrix}
        \mathbf{0} & \mathbf{0} & \mathbf{0}\\
        \mathbf{0} & \mathbf{0} & \mathbf{I}\\
        \Vector{G}{g}{}^{\wedge} & \mathbf{0} & \mathbf{0}
    \end{bmatrix} \in \R^{9 \times 9}, \\
    \mathbf{C}^{0} &= \begin{bmatrix}
        \mathbf{I} & \mathbf{0}
    \end{bmatrix} \in \R^{9 \times 18} . \label{eq:C0}
\end{align}

The output matrix $\mathbf{C}^{0}$ is constant.  
The (1,1) and (1,2) blocks of $\mathbf{A}_{t}^{0}$ are constant where the (1,2) block encodes a constant linear coupling of the bias states into the navigation states.
Note that for existing design methodologies, where the bias is added without an integrated equivariant geometry, the coupling in the (1,2) block of $\mathbf{A}_{t}^{0}$ will be time-varying and depends on the observer state. 
One way of understanding this difference is that the bias for the \ac{eqf} is estimated in a time-varying ``frame of reference'' adapted to the symmetry of the navigation states. 
Conversely, classical approaches, and the ``imperfect'' \ac{iekf}, estimate the bias states in a copy of the Lie-algebra that is not adapted to the symmetry.
This structure leads to the improved bias estimation results shown in the sequel. 
The (2,2) block of the matrix $\mathbf{A}_{t}^{0}$ for the \ac{eqf} encodes bias error dynamics induced by the time-variation of the body-frame. 

\subsection{\acl{eqf} ODEs}

To summarize, let ${\hat{X} = \left(\hat{A}, \hat{a}\right)}$ be the \acl{eqf} state with initial condition ${\hat{X}\left(0\right) = \left(I, 0\right)}$ and ${\bm{\Sigma} \in \PD\left(18\right) \subset \R^{18 \times 18}}$ be the Riccati matrix with initial condition ${\bm{\Sigma}\left(0\right) = \bm{\Sigma}_0}$. Let $\mathbf{A}_{t}^{0}$ and $\mathbf{C}^{0}$ be the matrices defined respectively in \equref{At0} and \equref{C0}. Let ${\Delta = \left(\Delta_1, \Delta_2\right)}$ be the innovation term. Let ${\bm{P} \in \PD\left(18\right) \subset \R^{18 \times 18}}$ and ${\bm{Q} \in \PD\left(9\right) \subset \R^{9 \times 9}}$ be respectively a state gain matrix and an output gain matrix. Therefore, the \acl{eqf} dynamics, the innovation term and the Riccati Dynamics are
\begin{subequations}
\label{eq:observer_eq}
\begin{align}
    &\dot{\hat{A}} = \hat{A}\left(\Vector{I}{w}{}^{\wedge} + \Adsym{\hat{A}^{-1}}{\hat{a}}\right) + \Vector{G}{g}{}^{\wedge}\hat{A} + f^{0}_1\left(\hat{A}\right) + \Delta_1\hat{A} , \label{eq:observerode1}\\
    \begin{split}
       &\dot{\hat{a}} = \adsym{-\hat{a}}{\left(\Adsym{\hat{A}}{\Vector{I}{w}{}^{\wedge}} + \hat{a}\right) + \Vector{G}{g}{}^{\wedge} + f^0_1\left(\hat{A}\right)}\\
       &\quad\; - \Adsym{\hat{A}}{\tau^{\wedge}} + \Delta_2 + \adsym{\Delta_1}{\hat{a}} ,
    \end{split} \label{eq:observerode2}\\
        &\Delta = \Fr{E}{\mathbf{I}}\phi_{\xi_0}\left(E\right)^{\dagger}\td\varepsilon^{-1}\bm{\Sigma}\mathbf{C}^{0^T}\bm{Q}^{-1}\delta\left(\rho_{\hat{X}^{-1}}\left(y\right)\right) , \label{eq:inno}\\
        & \dot{\bm{\Sigma}} = \mathbf{A}_{t}^{0}\bm{\Sigma} + \bm{\Sigma}{\mathbf{A}_{t}^{0}}^T + \bm{P} - \bm{\Sigma}{\mathbf{C}^{0}}^T\bm{Q}^{-1}\mathbf{C}^{0}\bm{\Sigma} .\label{eq:riccatiode}
\end{align}
\end{subequations}

\section{Stability Analysis}\label{sec:stab}

\begin{theorem}\label{th:observer}
Consider the observer \eqref{eq:observer_eq} computed for local coordinates~\eqref{eq:local_coord_xi} \eqref{eq:local_coord_y} and linearized model $\mathbf{A}^0_t$ and ${\mathbf{C}^{0}}$ given by \eqref{eq:At0} and \eqref{eq:C0}. 
Assume that the trajectory $\xi_t$ and the observer evolve such that the matrix pair $(\mathbf{A}^0_t, {\mathbf{C}^{0}})$ is uniformly observable. 
Then, there exists a local neighbourhood of $\xi_0 \in \calM$ such that for any initial condition of the system such that the initial error $e(0)$ lies in this neighbourhood, the observer \eqref{eq:observer_eq} is defined for all time and $e(t) \to \xi_0$ is locally exponentially stable.
\end{theorem}

\begin{remark}
Note that the assumption on uniform observability of the error dynamics linearisation $(\mathbf{A}^0_t, {\mathbf{C}^{0}})$ can be inferred from uniform observability of the state trajectory, at least in a local neighbourhood\cit{Song1992ExtendedSystems}. 
The details of such an extension are beyond the scope of the present article. 
\end{remark}
\section{Experiments}\label{sec:sim}

To validate the proposed \acl{eqf} design for a \acl{bins}, we performed different simulations, as well as, real-world experiments, of a \ac{uav} carrying an \ac{imu} and receiving extended pose measurements. In simulation, the \ac{uav} is navigating a spline interpolated trajectory passing through randomly chosen waypoints, providing ground-truth values for the \ac{uav} position, orientation, velocities (linear and angular), and acceleration. These values were used to generate simulated extended pose measurements as well as simulated \ac{imu} measurements. In order to replicate a real-world scenario, Gaussian noise was added to both the simulated extended pose (modeled using the left equation in \equref{real_meas}) and the \ac{imu} measurements. Moreover non zero, time-varying biases (modeled as random walk processes) were added to the simulated \ac{imu} measurements. The standard deviation of the measurement noise values are ${\sigma_{\Vector{}{w}{}} = 1.3\cdot10^{-2}}$\si[per-mode = symbol]{\radian\per\second}, and ${\sigma_{\Vector{}{a}{}} = 8.3\cdot10^{-2}}$\si[per-mode = symbol]{\meter\per\square\second} for simulated \ac{imu} measurements reflecting the real \ac{imu} noise specification onboard a AscTec Hummingbird quadcopter used in the real world experiments, and ${\sigma_{\Vector{}{\theta}{}} = 8.7\cdot10^{-2}}$\si[per-mode = symbol]{\radian}, ${\sigma_{\Vector{}{p}{}} = 0.25}$\si[per-mode = symbol]{\meter} and ${\sigma_{\Vector{}{v}{}} = 0.1}$\si[per-mode = symbol]{\meter\per\second} for extended pose measurements.
\begin{figure}[!t]
\centering
\includegraphics[width=.925\linewidth]{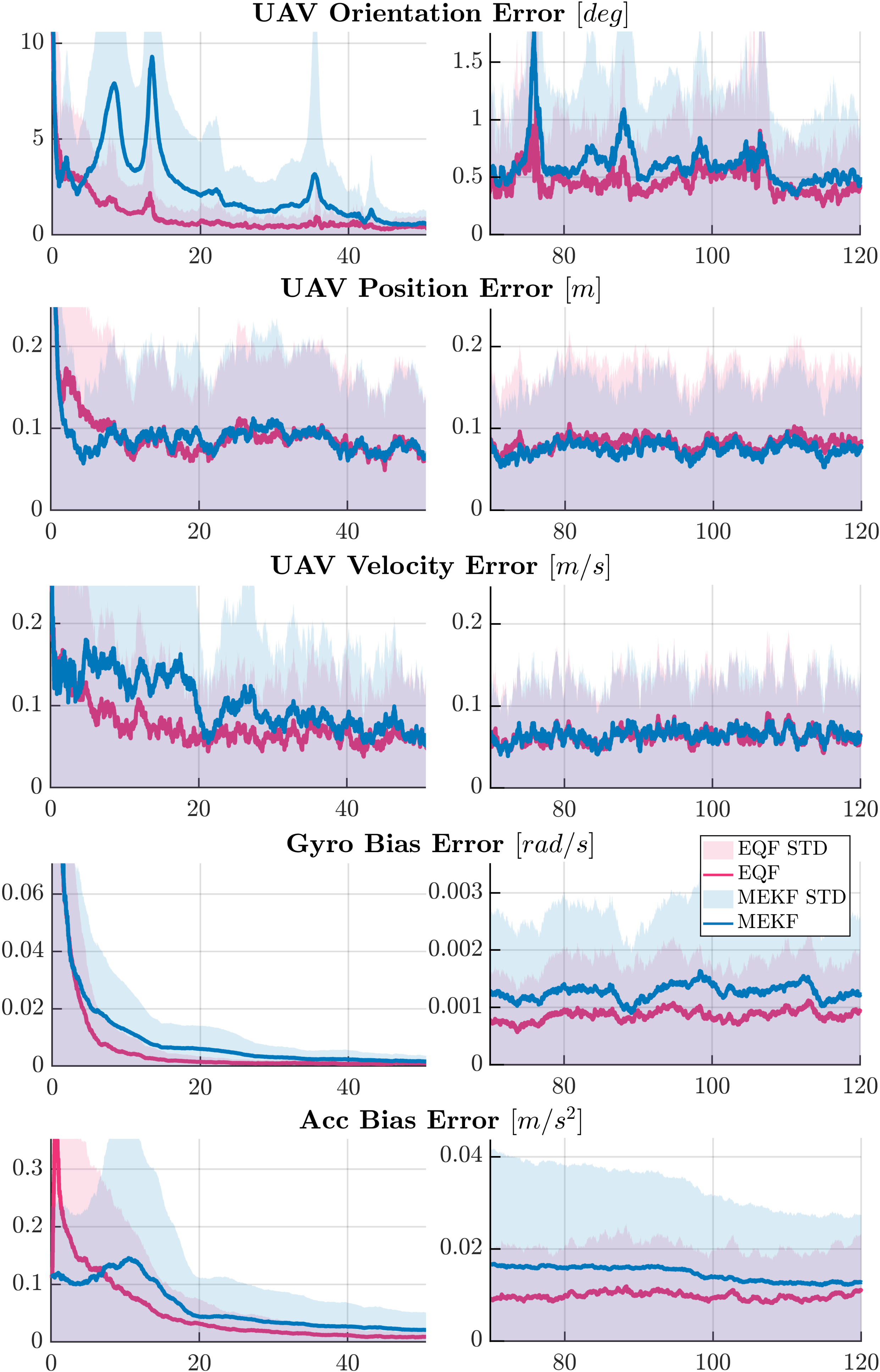}
\vspace{-5pt}
\caption{Averaged norm of the navigation and biases error over the 15 runs. \textbf{Left column:} Transient phase. \textbf{Right column:} Asymptotic phase. Note the different $y$-axes scale for the transient and the asymptotic phase.}
\label{fig:averaged_error}
\vspace{-10pt}
\end{figure}
\begin{table}[!t]
\renewcommand{\arraystretch}{1.15}
\setlength\tabcolsep{4.5pt}
\centering
\scriptsize
\caption{}
\vspace{-5pt}
\begin{tabular}{c c c c c c}
\specialrule{.1em}{.1em}{.1em} 
\textsc{Rmse} & Att [\si[per-mode = symbol]{\degree}]& Pos [\si[per-mode = symbol]{\meter}] & Vel [\si[per-mode = symbol]{\meter\per\second}] & bias [\si[per-mode = symbol]{\radian\per\second}] & bias [\si[per-mode = symbol]{\meter\per\square\second}] \\ \hline
\ac{eqf} $(T)$ & $\mathbf{2.3032}$ & $0.1884$ & $\mathbf{0.1480}$ & $\mathbf{0.0224}$ & $0.0937$ \\
\acs{mekf} $(T)$ & $4.7448$ & $\mathbf{0.1819}$ & $0.1751$ & $0.0241$ & $\mathbf{0.0899}$\\
\specialrule{.1em}{.1em}{.1em} 
\ac{eqf} $(A)$ & $\mathbf{0.5502}$ & $0.0902$ & $\mathbf{0.0698}$ & $\mathbf{0.0009}$ & $\mathbf{0.0104}$ \\
\acs{mekf} $(A)$ & $0.6896$ & $\mathbf{0.0815}$ & $0.0713$ & $0.0014$ & $0.0146$\\
Meas $(A)$ & $7.8687$ & $0.2496$ & $0.2505$ & $-$ & $-$\\
\specialrule{.1em}{.1em}{.1em} 
\end{tabular}
\label{tab:rmse}
\vspace{-10pt}
\end{table}

In all the simulations and real-world experiments, the state origin ${\xi_{0}}$ was chosen to be the identity of the state space, leading to an identity output origin $y_{0}$. The state and the output gain matrix for the \acl{eqf} were selected to reflect input and output measurements noise respectively. 


In the first simulation setup, the proposed \ac{eqf} and a MEKF were evaluated on a 15 runs Monte-Carlo simulation. Each run has different input data, covering a time span of \SI{120}{\second} with \ac{imu} provided at \SI{100}{\hertz}, and extended pose measurements provided at \SI{30}{\hertz}, to account for different sensor clock frequencies. Both filters were wrongly initialized with identity orientation, zero position, velocity, and biases. The initial value for the Riccati matrix was chosen to reflect the initial error. \tabref{rmse} reports the averaged \acs{rmse} over the 15 runs, of the navigation states and biases for the two filters, computed for the first \SI{40}{\second} (the transient, $T$) and the last \SI{40}{\second} (the asymptotic behaviour, $A$) of each run. Furthermore, \figref{averaged_error} better shows the trend of the averaged error, and the sample standard deviation, of the navigation and biases states. It is of particular interest to note the improved transient performance and the lower $\sim 30\%$ asymptotic error on the bias states, of the proposed \acs{eqf}.

To show the benefit on the stability of the appropriate geometric description of the nonlinear structure of the problem, we compared the \ac{eqf} and the MEKF in a known challenging condition for \ac{ekf}-like filters, in particular, in this simulation we consider a typical run with noise-free \ac{imu} and extended pose measurements provided respectively at \SI{100}{\hertz}, and \SI{30}{\hertz}. Both filters were wrongly initialized, and different tuning of the state gain matrix $\bm{P}$ were considered. An appropriate tight tuning $\bm{P}_T$ (small process noise), consistent with a highly accurate \ac{imu}, and a robust tuning $\bm{P}_R \gg \bm{P}_T$ with artificially inflated process noise\cit{Reif1998AnStability}. \tabref{rmse_tight_inflated} shows the \acs{rmse} of the transient phase for both filter tuning. With tight tuning the MEKF diverges (denoted as FAIL \tabref{rmse_tight_inflated}). This is due to the gain quickly approaching zero whereas the estimation error has not sufficiently decreased due to nonlinearities. The \ac{eqf}, instead, remains unaffected. A robust tune with inflated process noise is required to make the MEKF converge, however, this engineering trick often degrades the stochastic performance of the filter given that the tuning is not according to the sensors noise specifications.
\begin{figure}[!t]
\centering
\includegraphics[width=.925\linewidth]{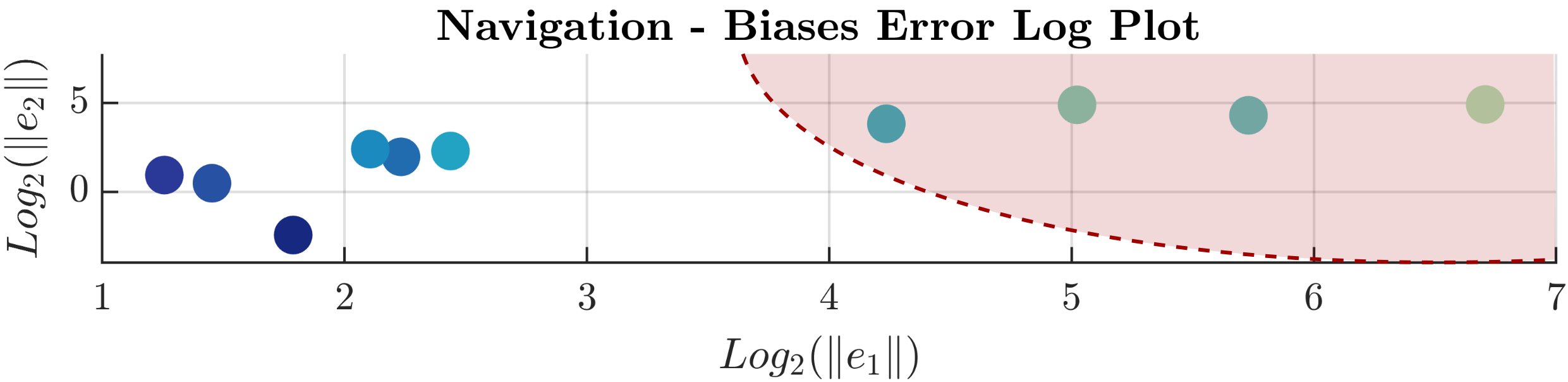}
\vspace{-5pt}
\caption{Initialization error in logarithmic scale. $e_1$ represent the navigation error, while $e_2$ represent the biases error. The red area represents the region of initial condition that lead to a failure of the MEKF, whereas the \ac{eqf} converges for each initial condition.}
\label{fig:logerrors}
\vspace{-10pt}
\end{figure}
\begin{table}[!t]
\renewcommand{\arraystretch}{1.15}
\setlength\tabcolsep{4.5pt}
\centering
\scriptsize
\caption{}
\vspace{-5pt}
\begin{tabular}{c c c c c c}
\specialrule{.1em}{.1em}{.1em} 
\textsc{Rmse} & Att [\si[per-mode = symbol]{\degree}]& Pos [\si[per-mode = symbol]{\meter}] & Vel [\si[per-mode = symbol]{\meter\per\second}] & bias [\si[per-mode = symbol]{\radian\per\second}] & bias [\si[per-mode = symbol]{\meter\per\square\second}] \\ \hline
\ac{eqf} $(\bm{P}_T, T)$ & $\mathbf{3.1313}$ & $\mathbf{0.1022}$ & $\mathbf{0.0530}$ & $\mathbf{0.0694}$ & $\mathbf{0.1331}$ \\
\acs{mekf} $(\bm{P}_T, T)$ & \color{red}FAIL & \color{red}FAIL & \color{red}FAIL & \color{red}FAIL & \color{red}FAIL\\
\specialrule{.1em}{.1em}{.1em}
\ac{eqf} $(\bm{P}_R, T)$ & $\mathbf{3.1584}$ & $0.1001$ & $\mathbf{0.0451}$ & $\mathbf{0.0695}$ & $0.1411$ \\
\acs{mekf} $(\bm{P}_R, T)$ & $6.0439$ & $\mathbf{0.0812}$ & $0.1362$ & $ 0.0896$ & $\mathbf{0.1367}$\\
\specialrule{.1em}{.1em}{.1em} 
\end{tabular}
\label{tab:rmse_tight_inflated}
\vspace{-10pt}
\end{table}

To show the robustness of the proposed methodology to different initial conditions, in the third simulation setup, we ran a batch of 10 different runs with the same input data including \ac{imu} and extended pose measurements provided respectively at \SI{100}{\hertz}, and \SI{30}{\hertz}, but with different initialization of the filters with increasing initialization error in all the states. The initial value for the Riccati matrix was increased accordingly with the initialization error. \figref{logerrors} shows the different initialization error in logarithmic scale. The proposed \ac{eqf} exhibit improved robustness to initialization error converging in each run of the experiment whereas the MEKF diverges when the initial error is high (red shaded area in \figref{logerrors}). Unknown true initial state is common in robotics application and improved stability and robustness to wrong initialization are desirable properties for an estimator.

Finally, the proposed \ac{eqf} was evaluated, and compared with a MEKF, on real-world in-flight data. Raw \ac{imu} readings were recorded at \SI{200}{\hertz} on a AscTec Hummingbird quadcopter flying a Lissajous curve\cit{Bohm2021}, and raw pose measurements gathered at $\sim$\SI{30}{\hertz} from a OptiTrack motion capture system. Velocity measurements were derived by numerical differentiation of the filtered position measurements. The obtained extended pose measurements at \SI{30}{\hertz} were then added with zero mean white Gaussian noise. Standard deviation of the \ac{imu}, and extended pose, measurement noise values was set to the aforementioned values. 
%
%
\tabref{rmse_L} shows the \acs{rmse} (on the full trajectory) of the \ac{eqf} compared to the \ac{ekf} for the navigation states on the real-world in-flight data. Note that there is no ground-truth for the bias states, hence we omitted the corresponding entries. 
\begin{table}[!ht]
\renewcommand{\arraystretch}{1.15}
\centering
\scriptsize
\caption{}
\vspace{-5pt}
\begin{tabular}{c c c c c}
\specialrule{.1em}{.1em}{.1em} 
\textsc{Rmse}  & Att [\si[per-mode = symbol]{\degree}]& Pos [\si[per-mode = symbol]{\meter}] & Vel [\si[per-mode = symbol]{\meter\per\second}] \\ \hline
\ac{eqf} & $\mathbf{3.4448}$ & $\mathbf{0.1210}$ & $\mathbf{0.5224}$ \\
\acs{mekf} & $3.9755$ & $0.1357$ & $0.5473$ \\
\specialrule{.1em}{.1em}{.1em}
\end{tabular}
\label{tab:rmse_L}
\vspace{-14pt}
\end{table}

\section{Conclusion}\label{sec:conc}

In this paper we presented a novel \acl{eqf} design for the \acl{bins}. We introduced an innovative input-bias extension to the original system formulation, together with a symmetry based on semi-direct product group, that lead to the proof of the system's equivariance.
In contrast to state-of-the-art solutions, we presented an equivariant observer for the concurrent estimation of navigation states and input measurement biases characterized by a well behaving linearized error dynamics, asymptotic stability, improved performance and robustness. The results shown in this paper could lead to a series of work that can potentially bridge the gap between existing geometric observers and practical inertial-based navigation robotic applications since, to the best of our knowledge, we are the first to show the inclusion of bias terms as present in \ac{imu}s, in a tightly-coupled geometric observer design.
We validated our theoretical findings and the proposed \ac{eqf} in extensive simulation setup and real-world in-flight data, specifically to show its improved accuracy, stability and robustness compared to industry standard solutions for robotics application.




\end{document}


\maketitle
\thispagestyle{empty}
\pagestyle{empty}


This document provides supplementary material for the ICRA 2022 submission ``Equivariant Filter Design for Inertial Navigation Systems with Input Measurement Biases''\cit{fornasier22eqf}.

\section*{Mathematical Preliminaries}

\setcounter{section}{1}

\subsection{Fréchet derivative}

Let ${h \AtoB{\calM}{\calN}}$ be a continuous and differentiable map between manifolds. 
The differential is writen ${\td h \AtoB{\tT\calM}{\tT\calN}}$. Given ${\xi \in \calM}$ and ${\eta_{\xi} \in \tT_{\xi}\calM}$, the differential $\td h$ is evaluated pointwise by the Fréchet derivative\cit{Coleman2012CalculusSpaces} as
\begin{equation*}
    \td h\left(\xi\right)\left[\eta_{\xi}\right] = \Fr{\zeta}{\xi}h\left(\zeta\right)\left[\eta_{\xi}\right] \in \tT_{h\left(\xi\right)}\calN .
\end{equation*}

\subsection{Left translation on the Semi-Direct product group}

Define the left translation on the semi-direct product group $\sdpgrpG$ by ${L_X : \sdpgrpG \to \sdpgrpG}$, ${L_X Y := XY}$.
Define a map ${\td L_X : \text{T}_Y \sdpgrpG \to \text{T}_{XY} \sdpgrpG}$ by
\begin{align*}
\td L_{(A,a)}[w_1, w_2] = (A w_1, \Ad_A [w_2]).
\end{align*}
\begin{lemma} \label{lem:dL}
$\td L_{(A,a)}$ is the differential of the left translation $L_{(A,a)}$.
\end{lemma}
\begin{proof}
Computing the differential of the left translation
\begin{align*}
\tD_{(B,b)}\Big|_{I,0} &((A,a)(B,b))[w_1, w_2] \\
&= \tD_{(B,b)}\Big|_{I,0} (AB, a + \Ad_A [b])[w_1, w_2] \\
&= (A w_1, \Ad_A [w_2]) \\
&= \td L_{(A,a)}[w_1, w_2].
\end{align*}
\end{proof}

\subsection{Right translation on the Semi-Direct product group}

Define the right translation on the semi-direct product group $\sdpgrpG$ by $R_X : \sdpgrpG \to \sdpgrpG$, $R_X Y := YX$.
Define a map $\td R_X : \text{T}_Y \sdpgrpG \to \text{T}_{YX} \sdpgrpG $ by
\begin{align*}
\td R_{(A,a)}[w_1, w_2] = (w_1 A, w_2 + \ad_{w_1} [a]).
\end{align*}
\begin{lemma} \label{lem:dR}
$\td R_{(A,a)}$ is the differential of the right translation $R_{(A,a)}$.
\end{lemma}
\begin{proof}
Computing the differential of the right translation
\begin{align*}
\tD_{(B,b)}\Big|_{I,0} &((B,b)(A,a))[w_1, w_2] \\
&= \tD_{(B,b)}\Big|_{I,0} (BA, b + \Ad_B [a])[w_1, w_2] \\
&= (w_1 A, w_2 + \ad_{w_1}[a]) \\
&= \td R_{(A,a)}[w_1, w_2].
\end{align*}
\end{proof}

\section*{Proofs}

This section provides extended proofs of the Theorems and Lemmas presented in the main document\cit{fornasier22eqf}. It should be noted that the references to Lemmas, Theorems, and Equations in this document refer to the original numbering from the main document\cit{fornasier22eqf}.

\subsection{Proof of Lemma 4.1}

\begin{manuallemma}{4.1}
Define ${\phi \AtoB{\SE_2(3)_{\se_2(3)}^{\ltimes} \times \calM}{\calM}}$ as
\begin{equation*}
   \phi\left(X, \xi\right) \coloneqq \left(\Pose{}{}A,\; \Adsym{A^{-1}}{\Vector{}{b}{}^{\wedge} - a}\right) .
\end{equation*}
Then, $\phi$ is a transitive right group action of $\SE_2(3)_{\se_2(3)}^{\ltimes}$ on $\calM$.
\end{manuallemma}
\begin{proof}
Let ${X, Y \in \SE_2(3)_{\se_2(3)}^{\ltimes}}$ and $\xi \in \calM$. Then, 
\begin{align*}
    &\phi\left(X,\phi\left(Y, \xi\right)\right)\\
    &\quad= \phi\left(X,\; \left(\Pose{}{}B,\; \Adsym{B^{-1}}{\Vector{}{b}{}^{\wedge} - b}\right)\right)\\
    &\quad= \left(\Pose{}{}BA,\; \Adsym{{A}^{-1}}{\Adsym{B^{-1}}{\Vector{}{b}{}^{\wedge} - b} - a}\right)\\
    &\quad= \left(\Pose{}{}BA,\; \Adsym{\left(BA\right)^{-1}}{\Vector{}{b}{}^{\wedge} - \left(b + \Adsym{B}{a}\right)}\right)\\
    &\quad= \phi\left(YX, \xi\right) .
\end{align*}
This shows that $\phi$ is a valid right group action. Then, ${\forall \; \xi_1, \xi_2 \in \calM}$ we can always write the group element ${Z = \left(\Pose{}{}_{1}^{-1}\Pose{}{}_{2}, \Vector{}{b}{1}^{\wedge} - \Adsym{\left(\Pose{}{}_{1}^{-1}\Pose{}{}_{2}\right)}{\Vector{}{b}{2}^{\wedge}}\right)}$, such that
\begin{align*}
    \phi\left(Z, \xi_1\right) &= \left(\Pose{}{}_{1}\Pose{}{}_{1}^{-1}\Pose{}{}_{2},\right.\\
    &\quad\; \left.\Adsym{\left(\Pose{}{}_{1}^{-1}\Pose{}{}_{2}\right)^{-1}}{\Vector{}{b}{1}^{\wedge} - \Vector{}{b}{1}^{\wedge} + \Adsym{\left(\Pose{}{}_{1}^{-1}\Pose{}{}_{2}\right)}{\Vector{}{b}{2}^{\wedge}}}\right)\\
    &= \left(\Pose{}{}_{2},\; \Vector{}{b}{2}^{\wedge}\right) = \xi_2 ,
\end{align*}
which demonstrates the transitive property of the group action.
\end{proof}

\subsection{Proof of Lemma 4.2}

\begin{manuallemma}{4.2}
Define ${\psi \AtoB{\SE_2(3)_{\se_2(3)}^{\ltimes} \times \vecL}{\vecL}}$ as
\begin{equation*}
    \begin{split}
        \psi\left(X,\bm{u}\right) 
        &\coloneqq \left(\Adsym{A^{-1}}{\Vector{}{w}{}^{\wedge} - a} + {f^{0}_1\left(A^{-1}\right)},\; \Vector{}{g}{}^{\wedge},\; \Adsym{A^{-1}}{\tau^{\wedge}}\right) ,
    \end{split}
\end{equation*}
where $f^0_1 : \SE_2(3) \rightarrow \tT \SE_2(3)$ is given by Equ. (5).  Then, $\psi$ is a right group action of $\SE_2(3)_{\se_2(3)}^{\ltimes}$ on $\vecL$.
\end{manuallemma}
\begin{proof}
Let ${X, Y \in \SE_2(3)_{\se_2(3)}^{\ltimes}}$ and $\bm{u} \in \vecL$. Then, 
\begin{align*}
    &\psi\left(X,\psi\left(Y, \bm{u}\right)\right)\\
    &= \psi\left(X,\left(\Adsym{B^{-1}}{\Vector{}{w}{}^{\wedge} - b} + f^{0}_1\left(B^{-1}\right),\; \Vector{G}{g}{}^{\wedge},\; \Adsym{B^{-1}}{\tau^{\wedge}}\right)\right)\\
    &= \left(\Adsym{A^{-1}}{\Adsym{B^{-1}}{\Vector{}{w}{}^{\wedge} - b} + f^{0}_1\left(B^{-1}\right) - a} + f^{0}_1\left(A^{-1}\right),\right.\\
    &\quad\;\, \left.\Vector{G}{g}{}^{\wedge},\;\Adsym{\left(BA\right)^{-1}}{\tau^{\wedge}}\right)\\
    &= \left(\Adsym{\left(BA\right)^{-1}}{\Vector{}{w}{}^{\wedge} -\left(b + \Adsym{B}{a}\right)} + \Adsym{A^{-1}}{f^{0}_1\left(B^{-1}\right)}\right.\\ 
    &\quad\;\, \left.+ f^{0}_1\left(A^{-1}\right),\;\Vector{G}{g}{}^{\wedge},\;\Adsym{\left(BA\right)^{-1}}{\tau^{\wedge}}\right) .
\end{align*}
Due to the right invariant property of the drift field $f^0$, it can be shown that
\begin{align*}
    \left(f^{0}_1\left(A^{-1}B^{-1}\right),\; \mathbf{0}^{\wedge}\right)
    & = \left(A^{-1}f^{0}_1\left(B^{-1}\right) + f^{0}_1\left(A^{-1}\right),\; \mathbf{0}^{\wedge}\right)\\
    &= \left(\Adsym{A^{-1}}{f^{0}_1\left(B^{-1}\right)} + f^{0}_1\left(A^{-1}\right),\; \mathbf{0}^{\wedge}\right). 
\end{align*}
Therefore,
\begin{align*}
    &\psi\left(X,\psi\left(Y, \bm{u}\right)\right)\\
    &\quad= \left(\Adsym{\left(BA\right)^{-1}}{\Vector{}{w}{}^{\wedge} -\left(b + \Adsym{B}{a}\right)}\right.\\
    &\qquad\;\, \left. + \left(A^{-1}f^{0}_1\left(B^{-1}\right) + f^{0}_1\left(A^{-1}\right)\right),\; \Vector{G}{g}{}^{\wedge},\; \Adsym{\left(BA\right)^{-1}}{\tau^{\wedge}}\right)\\
    &\quad= \psi\left(YX, \bm{u}\right) .
\end{align*}
Thus, proving that $\psi$ is a valid right group action.
\end{proof}

\subsection{Proof of Theorem 4.3}

\begin{manualtheorem}{4.3}
The \acl{bins} in Equ. (6) is equivariant under the actions $\phi$ in Equ. (9) and $\psi$ in Equ. (10) of the symmetry group $\SE_2(3)_{\se_2(3)}^{\ltimes}$. That is
\begin{equation*}
    f^0\left(\xi\right) + f_{\psi_{X}\left(\bm{u}\right)}\left(\xi\right) = \Phi_{X}f^0\left(\xi\right) + \Phi_{X}f_{\bm{u}}\left(\xi\right) .
\end{equation*}
\end{manualtheorem}
\begin{proof}
Let ${X \in \SE_2(3)_{\se_2(3)}^{\ltimes}}$, ${\xi \in \calM}$ and ${\bm{u} \in \vecL}$, then ${\phi_{X^{-1}}\left(\xi\right) = \left(\Pose{}{}A^{-1},\; \Adsym{A}{\Vector{}{b}{}^{\wedge}} + a\right)}$ and
$f^0\left(\phi_{X^{-1}}\left(\xi\right)\right) + f_{\bm{u}}\left(\phi_{X^{-1}}\left(\xi\right)\right) = \left(f^{0}_1\left(\mathbf{T}A^{-1}\right),\mathbf{0}^{\wedge}\right) + \left(\Pose{}{}A^{-1}\left(\Vector{}{w}{}^{\wedge} - \Adsym{A}{\Vector{}{b}{}^{\wedge}} - a\right) + \Vector{}{g}{}^{\wedge}\Pose{}{}A^{-1},\; \tau^{\wedge}\right)$, 
one has
\begin{align*}
    &\Phi_{X}f^0\left(\xi\right) + \Phi_{X}f_{\bm{u}}\left(\xi\right) \\
    &\quad:= \td \phi_{X} \left( 
    f^0\left(\phi_{X^{-1}}\left(\xi\right)\right) + f_{\bm{u}}\left(\phi_{X^{-1}}\left(\xi\right)\right) \right) \\ 
    &\quad =\left(\Pose{}{}f^{0}_1\left(A^{-1}\right) + f^{0}_1\left(\mathbf{T}\right), \mathbf{0}^{\wedge}\right)\\
    &\qquad+ \left(\Pose{}{}A^{-1}\left(\Vector{}{w}{}^{\wedge} - \Adsym{A}{\Vector{}{b}{}^{\wedge}} - a\right)A + \Vector{}{g}{}^{\wedge}\Pose{}{}A^{-1}A,\right.\\
    &\qquad\;\;\, \Adsym{A^{-1}}{\tau^{\wedge}}\Big)\\
    &\quad= \left(\Pose{}{}\left(\Adsym{A^{-1}}{\Vector{}{w}{}^{\wedge} - a} + f^{0}_1\left(A^{-1}\right) - \Vector{}{b}{}^{\wedge}\right) + \Vector{}{g}{}^{\wedge}\Pose{}{},\right.\\
    &\qquad\;\;\, \Adsym{A^{-1}}{\tau^{\wedge}}\Big) + \left(f^{0}_1\left(\mathbf{T}\right), \mathbf{0}^{\wedge}\right) = f^{0}\left(\xi\right) + f_{\psi_{X}\left(\bm{u}\right)}\left(\xi\right) .
\end{align*}
This proves the equivariance of the system.
\end{proof}

\subsection{Proof of Lemma 4.5}

\begin{manuallemma}{4.5}
Define ${\rho \AtoB{\SE_2(3)_{\se_2(3)}^{\ltimes} \times \calN}{\calN}}$ as
\begin{equation}
    \rho\left(X,y\right) \coloneqq yA .
\end{equation}
Then, the configuration output defined in Equ. (7) is equivariant\cit{VanGoor2020EquivariantSpaces}.
\end{manuallemma}
\begin{proof}
    Let ${X = \left(A,a\right) \in \SE_2(3)_{\se_2(3)}^{\ltimes}}$ and $h\left(\xi\right) = \Pose{}{} \in \calN$ then,
    \begin{equation}
        \rho\left(X, h\left(\xi\right)\right) = \Pose{}{}A = h\left(\phi\left(X,\xi\right)\right) .
    \end{equation}
    This proves the output equivariance. 
\end{proof}

\subsection{Proof of Theorem 5.1}

\begin{manualtheorem}{5.1}
Define ${\Lambda_1 \AtoB{\calM \times \vecL} \se_2(3)}$ as
\begin{equation}\label{eq:lift1}
    \Lambda_1\left(\xi, \bm{u}\right) \coloneqq \left(\Vector{}{w}{}^{\wedge} - \Vector{}{b}{}^{\wedge}\right) + \Adsym{\Pose{}{}^{-1}}{\Vector{}{g}{}^{\wedge}} + \Pose{}{}^{-1}f^0_1\left(\mathbf{T}\right) .
\end{equation}
And, define ${\Lambda_2 \AtoB{\calM \times \vecL} \se_2(3)}$ as
\begin{equation}\label{eq:lift2}
    \Lambda_2\left(\xi, \bm{u}\right) \coloneqq \adsym{\Vector{}{b}{}^{\wedge}}{\Lambda_1\left(\xi, \bm{u}\right)} - \tau^{\wedge} .
\end{equation}
Then, the map ${\Lambda\left(\xi, \bm{u}\right) = \left(\Lambda_1\left(\xi, \bm{u}\right),\; \Lambda_2\left(\xi, \bm{u}\right)\right)}$ is an equivariant lift for the system in Equ. (6) with respect to the defined symmetry group.
\end{manualtheorem}
\begin{proof}
Let ${\xi \in \calM}, \bm{u} \in \vecL$ and ${X \in \se_2(3)_{\se_2(3)}^{\ltimes}}$ then
\begin{align*}
    &d\phi_{\xi}\left(\mathbf{I}\right)\left[\Lambda_1\left(\xi, \bm{u}\right), \Lambda_2\left(\xi, \bm{u}\right)\right]\\
    &\quad= \left(\Pose{}{}\Lambda_1\left(\xi, \bm{u}\right), -\adsym{\Lambda_1\left(\xi, \bm{u}\right)}{\Vector{}{b}{}^{\wedge}} - \Lambda_2\left(\xi, \bm{u}\right)\right)\\
    &\quad= \left(\Pose{}{}\left(\Vector{}{w}{}^{\wedge} - \Vector{}{b}{}^{\wedge}\right) + \Vector{}{g}{}^{\wedge}\Pose{}{} + f^0_1\left(\mathbf{T}\right),\right.\\
    &\qquad\;\left. -\adsym{\Lambda_1\left(\xi, \bm{u}\right)}{\Vector{}{b}{}^{\wedge}} - \adsym{\Vector{}{b}{}^{\wedge}}{\Lambda_1\left(\xi, \bm{u}\right)} + \tau^{\wedge}\right)\\
    &\quad=  \left(\Pose{}{}\left(\Vector{}{w}{}^{\wedge} - \Vector{}{b}{}^{\wedge}\right) + \Vector{}{g}{}^{\wedge}\Pose{}{} + f^0_1\left(\mathbf{T}\right),\; \tau^{\wedge}\right)\\
    &\quad= f^{0}\left(\xi\right) + f\left(\xi, \bm{u}\right) ,
\end{align*}
where we have made use of the anti-commutative property of the Lie bracket. To demonstrate the equivariance of the lift we have to show that the condition ${\Adsym{X}{\Lambda\left(\phi_{X}\left(\xi\right),\psi_{X}\left(\bm{u}\right)\right)} = \Lambda\left(\xi, \bm{u}\right)}$ holds. Let ${\xi \in \calM}, \bm{u} \in \vecL$ and ${X \in \SE_2(3)_{\se_2(3)}^{\ltimes}}$. The adjoint of the semi-direct product follows
\begin{align*}
    &\Adsym{X}{\Lambda\left(\phi_{X}\left(\xi\right),\psi_{X}\left(\bm{u}\right)\right)}\\ &\quad=\Big(\Adsym{A}{\Lambda_1\left(\phi_{X}\left(\xi\right),\psi_{X}\left(\bm{u}\right)\right)},\Omega\Big) ,
\end{align*}
where $\Omega$ is defined as
\begin{equation*}
    \Omega = \Adsym{A}{\Lambda_2\left(\phi_{X}\left(\xi\right),\psi_{X}\left(\bm{u}\right)\right)} - \adsym{\Adsym{A}{\Lambda_1\left(\phi_{X}\left(\xi\right),\psi_{X}\left(\bm{u}\right)\right)}}{a} .
\end{equation*}
Then, we have
\begin{align*}
    &\Adsym{A}{\Lambda_1\left(\phi_{X}\left(\xi\right),\psi_{X}\left(\bm{u}\right)\right)}\\
    &= \mathrm{Ad}\left(A\right)\left[\Adsym{A^{-1}}{\Vector{}{w}{}^{\wedge} - a} + f^0_1\left(A^{-1}\right) - \Adsym{A^{-1}}{\Vector{}{b}{}^{\wedge} - a}\right.\\
    &\;\;\;\left. + \Adsym{A^{-1}}{\Adsym{\Pose{}{}^{-1}}{\Vector{}{g}{}^{\wedge}}} + A^{-1}\Pose{}{}^{-1}f^0_1\left(\mathbf{T}A\right)\right]\\
    &=\left(\Vector{}{w}{}^{\wedge} - \Vector{}{b}{}^{\wedge}\right) + \Adsym{\Pose{}{}^{-1}}{\Vector{}{g}{}^{\wedge}} + \Pose{}{}^{-1}f^0_1\left(\mathbf{T}A\right) + A f^0_1\left(A^{-1}\right)\\
    &=\left(\Vector{}{w}{}^{\wedge} - \Vector{}{b}{}^{\wedge}\right) + \Adsym{\Pose{}{}^{-1}}{\Vector{}{g}{}^{\wedge}} + \Pose{}{}^{-1}f^0_1\left(\mathbf{T}\right) = \Lambda_1\left(\xi,\bm{u}\right) ,
\end{align*}
where we have employed the right invariant property of $f^{0}$ and the fact that ${A f^0_1\left(A^{-1}\right) = -f^0_1\left(A\right)}$. Moreover,
\begin{align*}
    \Omega &= \Adsym{A}{\Lambda_2\left(\phi_{X}\left(\xi\right),\psi_{X}\left(\bm{u}\right)\right)} - \adsym{\Lambda_1\left(\xi,\bm{u}\right)}{a}\\
    &= \Adsym{A}{-\adsym{\Lambda_1\left(\phi_{X}\left(\xi\right),\psi_{X}\left(\bm{u}\right)\right)}{\Adsym{A^{-1}}{\Vector{}{b}{}^{\wedge} - a}}}\\
    &\;\;\;\, - \tau^{\wedge} - \adsym{\Lambda_1\left(\xi,\bm{u}\right)}{a}\\
    &= -\adsym{\Adsym{A}{\Lambda_1\left(\phi_{X}\left(\xi\right),\psi_{X}\left(\bm{u}\right)\right)}}{\Vector{}{b}{}^{\wedge} - a}\\
    &\;\;\;\,- \tau^{\wedge} - \adsym{\Lambda_1\left(\xi,\bm{u}\right)}{a}\\
    &= \adsym{\Vector{}{b}{}^{\wedge}}{\Lambda_1\left(\xi,\bm{u}\right)} - \tau^{\wedge} = \Lambda_2\left(\xi,\bm{u}\right) .
\end{align*}
\end{proof}

\subsection{Proof of Theorem 7.1}

\begin{manualtheorem}{7.1}
Consider the observer in Equ. (27) computed for local coordinates in Equ. (18), (21) and linearized model $\mathbf{A}^0_t$ and ${\mathbf{C}^{0}}$ given by Equ. (24) and (26). 
Assume that the trajectory $\xi_t$ and the observer evolve such that the matrix pair $(\mathbf{A}^0_t, {\mathbf{C}^{0}})$ is uniformly observable. 
Then, there exists a local neighbourhood of $\xi_0 \in \calM$ such that for any initial condition of the system such that the initial error $e(0)$ lies in this neighbourhood, the observer (27) is defined for all time and $e(t) \to \xi_0$ is locally exponentially stable.
\end{manualtheorem}
\begin{proof}
Assume that the initial condition for the state and the observer is such that the error lies in the local coordinate neighbourhood $\cal{U}_{\mr{\xi}}$ Equ. (18).
By continuity, there exists a time $T > 0$ (possibly infinite) on which the observer solution is well defined, continuous and the error remains in $\calU$ for $t \in [0,T)$.
Since $(\mathbf{A}^0_t, {\mathbf{C}^0})$ are uniformly observable and the gain matrices are bounded strictly-positive matrices, then there exists constants $0< \mathbf{\sigma}_1 < \mathbf{\sigma}_2 <\infty$ such that the solution of the Riccati equation Equ. (27d) satisfies \cite{1967_bucy_GlobalTheoryRiccati} $\mathbf{\sigma}_1 I_m < \mathbf{\Sigma}(t) < \mathbf{\sigma}_2 I_m$
for $t \in [0,T)$. 
Define $\calL_t =  \varepsilon^\top \mathbf{\Sigma}^{-1} \varepsilon.$
Consider the linearized error dynamics Equ. (19) and Equ. (22) on the interval $[0,T)$. 
One has
\begin{align*}
\ddt \calL_t 
%
&= -  \varepsilon^\top \left( \mathbf{\Sigma}^{-1} \mathbf{P} \mathbf{\Sigma}^{-1} + {\mathbf{C}^0}^\top \mathbf{Q}^{-1} {\mathbf{C}^0} \right)  \varepsilon .
\end{align*}
Since the right-hand side is negative definite, then there exists a neighbourhood of $0$ in the linearised error coordinates $\epsilon$ (corresponding to a neighbourhood $\calU_1$ of $\mr{\xi}$ in the error coordinates $e$) such that $\ddt \calL_t < 0$ for the full nonlinear error dynamics. 
Since $\mathbf{\Sigma}(t)$ is bounded below and above, there exists a value $\calL_0 > 0$ such that $\calL_t^{-1}([0,\calL_0]) \subset \calU_{\mr{\xi}} \cap \calU_1$ as a function of $\varepsilon$ for all times $t$.
It follows that for any initial condition such that the error $\varepsilon (e) \in \calL_t^{-1}([0,\calL_0])$ then $\calL(t) < \calL_0$ for all time and the solution exists and the error remains in $ \calU_{\mr{\xi}} \cap \calU_1$ for all time.
Exponential stability of the linearised error $\varepsilon$ follows directly from Lyapunov's second method and local exponential stability of the error $e$ for the full nonlinear system follows from the standard properties of linearisation and local coordinates.
\end{proof}

